\documentclass{article}

\usepackage[final]{nips_2017}

\usepackage[utf8]{inputenc} 
\usepackage[T1]{fontenc}    
\usepackage{hyperref}       
\usepackage{url}            
\usepackage{booktabs}       
\usepackage{amsfonts}       
\usepackage{amsmath}
\usepackage{amsthm}
\usepackage{amssymb}
\usepackage{nicefrac}       
\usepackage{microtype}      
\usepackage{enumitem}
\usepackage{alias}
\usepackage{mdframed}
\usepackage{wrapfig}
\usepackage{subcaption}

\title{Distributionally Robust Counterfactual Risk Minimization}
\author{Louis Faury\thanks{Equal contribution}  \\ Criteo AI Labs\\LTCI, Telecom ParisTech\\
  \texttt{l.faury@criteo.com} \And Ugo
  Tanielian$^*$\\ Criteo AI Labs \\LPSM, Universit\'e Paris 6\\
  \texttt{u.tanielian@criteo.com} \AND Flavian Vasile \\ Criteo AI Labs \\
  \texttt{f.vasile@criteo.com}
  \And Elena Smirnova \\ Criteo AI Labs \\ \texttt{e.smirnova@criteo.com} \And Elvis Dohmatob \\ Criteo AI Labs \\
  \texttt{e.dohmatob@criteo.com}
  }

\begin{document}

\maketitle

\begin{abstract}
This manuscript introduces the idea of using Distributionally Robust Optimization (DRO) for the Counterfactual Risk Minimization (CRM) problem. Tapping into a rich existing literature, we show that DRO is a principled tool for counterfactual decision making. We also show that well-established solutions to the CRM problem like sample variance penalization schemes are special instances of a more general DRO problem. In this unifying framework, a variety of distributionally robust counterfactual risk estimators can be constructed using various probability distances and divergences as uncertainty measures. We propose the use of Kullback-Leibler divergence as an alternative way to model uncertainty in CRM and derive a new robust counterfactual objective. In our experiments, we show that this approach outperforms the state-of-the-art on four benchmark datasets, validating the relevance of using other uncertainty measures in practical applications.
\end{abstract}

\label{sec:intro}
Learning how to act from historical data is a largely studied field in machine learning \cite{strehl2010learning,dudik2011doubly,li2011unbiased,li2015counterfactual}, spanning a wide range of applications where a system interacts with its environment (e.g search engines, ad-placement and recommender systems). Interactions are materialized by the actions taken by the system, themselves rewarded by a feedback measuring their relevance. Both quantities can be logged at little cost, and subsequently used to improve the performance of the learning system. The Batch Learning from Bandit Feedback \cite{swaminathan2015batch,swaminathan2015self}  (BLBF) framework describes such a situation, where a \emph{contextual} decision making process must be improved based on the logged history of \emph{implicit} feedback observed only on a subset of actions. Counterfactual estimators \cite{bottou2013counterfactual} allow to forecast the performance of any system from the logs, as if it was taking the actions by itself. This enables the search for an optimal system, even with observations biased towards actions favored by the logger. 

A natural approach to carry out this search consists in favoring systems that select actions with high empirical counterfactual rewards. However, this initiative can be rather burdensome as it suffers a crucial caveat intimately linked with a phenomenon known as the optimizer's curse \cite{capen1971competitive,smith2006optimizer,thaler2012winner}. It indicates that sorting actions by their empirical reward average can be sub-optimal since the resulting expected \emph{post-decision surprise} is non-zero. In real life applications where the space of possible actions is often extremely large and where decisions are taken based on very low sample sizes, the consequence of this phenomenon can be dire and motivates the design of principled \emph{robust} solutions. As a solution for this Counterfactual Risk Minimization (CRM) problem, the authors of \cite{swaminathan2015batch} proposed a modified action selection process, penalizing behaviors resulting in high-variance estimates.

In this paper, we argue that another line of reasoning resides in using Distributional Robust Optimization (DRO) for the CRM. It has indeed proven to be a powerful tool both in decision theory \cite{duchi2016,esfahani2018data,blanchet2016} and the training of robust classifiers \cite{madry2017towards,Xiao2018GeneratingAE,hu2018does}. Under the DRO formulation, one treats the empirical distribution with skepticism and hence seeks a solution that minimizes the worst-case expected cost over a family of distributions, described in terms of an uncertainty ball. Using distributionally robust optimization, one can therefore control the magnitude of the post-decision surprise, critical to the counterfactual analysis.  

We motivate the use of DRO for the CRM problem with asymptotic guarantees and bring to light a formal link between the variance penalization solution of~\cite{swaminathan2015batch} and a larger DRO problem for which the uncertainty set is defined with the chi-square divergence. Building from this, we propose the use of other uncertainty sets and introduce a KL-based formulation of the CRM problem. We develop a new algorithm for this objective and benchmark its performance on a variety of real-world datasets. We analyze its behavior and show that it outperforms existing state-of-the-art methods.

The structure of the paper is the following: in Section~\ref{sec:blbf} we formally introduce the BLBF framework and the CRM problem. In
Section~\ref{sec:dro} we present the DRO framework, motivate it for CRM and re-derive the POEM \cite{swaminathan2015batch} algorithm as one of its special cases and introduce a new CRM algorithm. In Section~\ref{sec:exp} we compare this new algorithm with state-of-the-art CRM algorithms on four public datasets and finally summarize our findings in Section~\ref{sec:conclusion}.

\section{Batch Learning from Logged Bandit Feedback}
\label{sec:blbf}

\subsection{Notation and terminology}
We denote $x\in\mX$ arbitrary \emph{contexts} drawn from an unknown distribution $\nu$ and presented to the decision maker. Such a quantity can describe covariate information about a patient for a clinical test, or a potential targeted user in a recommender system. The variable $y\in\mY$ denotes the \emph{actions} available to the decision maker - the potential medications to give to the patient, or possible advertisements targeting the user for instance. A \emph{policy} is a mapping $\pi:\mX\to\mP\left(\mY\right)$ from the space of contexts to probabilities in the action space. For a given (context, action) pair $(x,y)$, the quantity $\pi(y\vert x)$ denotes the probability of the policy $\pi$ to take the action $y$ when presented with the context $x$. When picking the action $y$ for a given context $x$, the decision maker receives a \emph{reward} $\delta(x,y)$, drawn from an unknown distribution. In our examples, this reward could indicate a patient's remission, or the fact that the targeted user clicked on the displayed ad. This reward $\delta(x,y)$ can also be assumed to be deterministic - as in \cite{swaminathan2015batch, swaminathan2015batch, swaminathan2015self}. We make this assumption in the rest of this manuscript. Finally, for a given context $x \in \mathcal{X}$ and an action $y \in \mathcal{Y}$, we define the cost function $c(x,y) \triangleq - \delta(x,y)$.

In this paper, we try to find policies producing low expected costs. To make this search tractable, it is usual to restrict the search to a family of \emph{parametric} policies, henceforth tying policies $\pi_\theta$ to a vector $\theta\in\Theta$. The $\emph{risk}$ $R(\theta) \triangleq \bE_{x\sim \nu, y\sim \pi_\theta(\cdot\vert x)}\left[ c(x,y) \right] $of the policy $\pi_\theta$ corresponds to the expected cost obtained by the policy $\pi_\theta$, a quantity the decision maker will try to minimize.

\subsection{Counterfactual Risk Minimization}
\label{subsec:crm}
In practical applications, it is common that one has only access to the \emph{interaction logs} of a previous version of the decision making system, also called a \emph{logging policy} (denoted $\pi_0$). More formally, we are interested in the case where the only available data is a collection of quadruplets $\mH\triangleq \left(x_i ,y_i, p_i, c_i\right)_{1\leq i \leq n}$, where the costs $c_i \triangleq c(x_i, y_i)$ were obtained after taking an action $y_i$ with probability $p_i\triangleq \pi_0(y_i\vert x_i)$ when presented with a context $x_i \sim \nu$.

In order to search for policies $\pi_\theta$ with smaller risk than $\pi_0$, one needs to build counterfactual estimators for $R(\theta)$ from the historic $\mH$. One way to do so is to use \emph{inverse propensity scores} \cite{rosenblum1983central}:
\begin{eqnarray}
    R(\theta) &= \bE_{x\sim \nu, y\sim \pi_\theta(x)}\left[ c(x,y) \right]
                    =  \bE_{x \sim \nu, y \sim \pi_0(x)}\left[ c(x,y) \frac{\pi_\theta(y\vert x)}{\pi_0(y\vert x)} \right],
\label{eq:ips_risk} 
\end{eqnarray}
for any $\pi_\theta$ absolutely continuous w.r.t $\pi_0$.
Henceforth, $R(\theta)$ can readily be approximated with samples $(x_i, y_i, p_i,c_i)$ from the interaction logs $\mH$ via the sample average approximation:
\begin{equation}
    R(\theta) \simeq \frac{1}{n}\sum_{i=1}^n c_i\frac{\pi_\theta(y_i\vert x_i)}{p_i}.
    \label{eq:ips_empirical_risk}
\end{equation}

Bluntly minimizing the objective provided by the counterfactual risk estimator \eqref{eq:ips_empirical_risk} is known to be sub-optimal, as it can have \emph{unbounded} variance \cite{ionides2008truncated}. It is therefore a classical technique \cite{bottou2013counterfactual,cortes2010learning,strehl2010learning,swaminathan2015batch} to \emph{clip} the propensity weights . This leads to the Clipped Inverse Propensity Scores (CIPS) estimator:
\begin{align}
    \hat{R}_n(\theta) \triangleq \frac{1}{n}\sum_{i=1}^n c_i \min\left(M, \frac{\pi_\theta(y_i\vert x_i)}{p_i}\right).
    \label{eq:empirical_clipped}
\end{align}
The variable $M$ is an hyper-parameter, balancing the variance reduction brought by weight clipping and the bias introduced in the empirical estimation of $R(\theta)$. The search for a minimal $\theta$ with respect to the estimator $\hat{R}_n^M(\theta)$ is often referred to as \emph{Counterfactual Risk Minimization} (CRM). 

\paragraph{Remark 1} Other techniques have been proposed in the literature to reduce the variance of the risk estimators. For instance, the doubly robust risk \cite{dudik2011doubly} takes advantage of both counterfactual estimators and supervised learning methods, while the self-normalized risk \cite{swaminathan2015self} was designed to counter the effect of a phenomenon known as propensity overfitting. We do not explicitly cover them in our analysis, but the results we derive hereinafter also hold for such estimators.


\subsection{Sample-Variance Penalization}

The main drawback of the CIPS estimator is that two different policies can have risk estimates of highly different variance - something the sample average approximation cannot capture. The authors of \cite{swaminathan2015batch} developed a variance-sensitive action selection process, penalizing policies with high-variance risk estimates. Their approach is based on a sample-variance penalized version of the CIPS estimator:
\begin{align}
    \hat{R}_n^{\lambda} (\theta) \triangleq \hat{R}_n(\theta) + \lambda \sqrt{V_n(\theta)/n},
    \label{eq:poem_obj}
\end{align}
where $\lambda$ is an hyper-parameter set by the practitioner, and $V_n(\theta)$ denotes the empirical variance of the quantities $c_i \min\left(M, \frac{\pi_\theta(y_i \vert x_i)}{p_i}\right)$. The main motivation behind this approach is based on confidence bounds derived in \cite{maurer09}, upper-bounding with high-probability the true risk $R(\theta)$ by the empirical risk $\hat{R}_n(\theta)$ augmented with an  additive empirical variance term. In a few words, this allows to build and optimize a pessimistic envelope for the true risk and penalize policies with high variance risk estimates. The authors of \cite{swaminathan2015batch} proposed the Policy Optimization for Exponential Models (POEM) algorithm and showed state-of-the-art results on a collection of counterfactual tasks when applying this method to exponentially parametrized policies:
\begin{eqnarray}
    \pi_\theta(y\vert x) \propto \exp\left( \theta^T \phi(x, y)\right),
    \label{eq:exponential_model}
\end{eqnarray}
with $\phi(x, y)$ a $d$-dimensional joint feature map and $\Theta$ a subset of $\mathbb{R}^d$.

\section{Distributionally Robust Counterfactual Risk Minimization}
\label{sec:dro}


\subsection{Motivating distributional robustness for CRM}

For more concise notations, let us introduce the variable $\xi=(x,y)$, the distribution $P=\nu\otimes \pi_0$ and the loss $\ell(\xi,\theta)\triangleq c(x,y)\min\left(M, \frac{\pi_\theta(y \vert x)}{\pi_0(y \vert x)}\right)$. The minimization of the counterfactual risk now writes $\theta^\star \triangleq \argmin_{\theta\in\Theta}\bE_{\xi\sim P}\left[\ell(\xi,\theta)\right]$ and its empirical counterpart $\hat{\theta}_n \triangleq \argmin_{\theta\in\Theta}\hat{R}_n(\theta)$. The consistency of the estimator $\hat{\theta}_n$ holds under general conditions and the empirical risk converges to the optimal true risk \cite{vapnik1992principles}:
\begin{equation}
     \mathbb E_{\xi \sim  P}[\ell(\xi; \theta^\star)] - \hat{R}_n(\hat{\theta}_n) \underset{n\to\infty}{\longrightarrow} 0
\end{equation}
However, one of the major drawbacks of this approach is that the empirical risk $\hat{R}_n(\theta)$ cannot be used as a \emph{performance certificate} for the true risk $\mathbb E_{\xi \sim P}[\ell(\xi;\theta)]$. Indeed, one will fail at controlling the true risk of any parameter $\theta$ since by the Central Limit Theorem:

\begin{equation}\label{eq:erm_drawback}
    \lim_{n \to \infty}\bP\left(\bE_{\xi \sim  P}[\ell(\xi;\theta)] \le \hat{R}_n(\theta)\right) = 1/2
\end{equation}

One way to circumvent this limitation is to treat the empirical distribution $\Pn$ with \emph{skepticism} and to replace it with an uncertainty set $\mU_\epsilon(\Pn)$ of distributions around $\Pn$ with $\epsilon > 0$ a parameter controlling the size of the uncertainty set $\mathcal U_\epsilon(\Pn)$. This gives rise to the \emph{distributionally robust} counterfactual risk:
\begin{align}
\Rrob{\mU}{\theta, \epsilon} \triangleq \max_{Q\in\mU_\epsilon(\Pn)} \bE_{\xi \sim Q}[\ell(\xi;\theta)].
\label{eq:robust_risk}
\end{align}
Minimizing this quantity w.r.t to $\theta$ yields the general DRO program: 
\begin{equation}
\begin{aligned}
  \Trob{} &\triangleq \argmin_{\theta \in \Theta} \Rrob{\mU}{\theta, \epsilon} \\&= \underset{\theta \in \Theta}{\argmin} \ \max_{Q \in \mathcal U_\epsilon(\Pn)}\mathbb E_{\xi \sim Q}[\ell(\xi;\theta)].
    \label{eq:dro}
\end{aligned}
\end{equation}

There is liberty on the way to construct the uncertainty set $\mathcal U_\epsilon(\Pn)$ including parametric \cite{madry2017towards,Xiao2018GeneratingAE} and non-parametric designs \cite{parys17,sinha2018,blanchet2016}. Moreover, for well chosen uncertainty sets \cite{duchi2016}, one can prove performance guarantees asserting that asymptotically (in the limit $n \rightarrow \infty$), for all $\theta\in\Theta$:
$$
\begin{aligned}
 &\bE_{\xi \sim  P}[\ell(\xi;\theta)] \le \Rrob{\mU}{\theta, \epsilon_n} \,\text{w.h.p}\\ \text{and} \quad
   &\bE_{\xi \sim  P}[\ell(\xi; \theta)] - \Rrob{\mU}{\theta, \epsilon_n} \rightarrow 0
\end{aligned}
$$
The robust risk therefore acts as a consistent certificate on the true risk. We believe that these properties alone are enough to motivate the use of DRO for the CRM problem, as it provides an elegant way to \emph{design consistent asymptotic upper bounds for the true risk} and ensure a small post-decision surprise. We detail such guarantees in the next subsection. Later, we draw links between DRO and the POEM, showing that a wide variety of DRO problems \emph{account for the empirical variance} of the samples, therefore mitigating the limitations of empirical averages as discussed in Section \ref{subsec:crm}.

\subsection{Guarantees of robustified estimators with $\varphi$-divergences}\label{section:3.2}
We are interested in DRO instances that are amenable to direct optimization. To this end, we focus here only on uncertainty sets $\mathcal U_\epsilon(\Pn)$ based on information divergences \cite{csiszar1967information}, since they strike a nice compromise between ease of implementation and theoretical guarantees that will reveal useful for the CRM problem. The use of information divergences for DRO has already been largely studied in several works (for example \cite{duchi2016,gotoh2018}). For the sake of completeness, we now recall in Definition~\ref{def:inf_divergences} the definition of information divergences.

\begin{definition}[$\varphi$-divergences]
  Let $\varphi$ be a real-valued, convex function such that $\varphi(1)=0$. For a reference distribution $P$, the
  \textit{divergence} of another distribution $Q$ with respect to $P$ is defined by
  \begin{align}
    D_\varphi(Q\|P) \triangleq \begin{cases}\int\varphi(dQ/dP)dP,&\mbox{ if }Q \ll
      P,\\+\infty,&\mbox{ else.}\end{cases}.
  \end{align}
  \label{def:inf_divergences}
\end{definition}

Subsequently, the definition of the uncertainty set $\mU_\epsilon$ relies only on $\varphi$-divergences as follows: 
\begin{align}
    \mathcal U_{\epsilon}(\hat{P}_n) = \left\{Q \mid D_\varphi(Q\|\Pn) \le \epsilon \right\}.
\end{align}

We need to ensure that the set of $\varphi$-divergences used to define the resulting robust risk $\Rrob{^{\varphi}}{\theta, \epsilon}$ satisfies some basic \emph{coherence} properties. We therefore make further assumptions about the measure of risk $\varphi$ to narrow the space of information divergences we consider: 

\begin{ass}[Coherence]
 $\varphi$ is a real-valued function satisfying:
\begin{itemize}
    \item $\varphi$ is convex and lower-semi-continuous
    \item $\varphi(t) = \infty$ for $t < 0$, and $\varphi(t) \ge \varphi(1)=0$, \, $\forall t\in\mathbb R$ 
    \item $\varphi$ is twice continuously differentiable at $t=1$ with
    $\varphi'(1)=0$ and $\varphi''(1)>0$
\end{itemize}
\label{ass:coherence}
\end{ass}

The axioms presented in Assumption~\ref{def:inf_divergences} have been proposed and studied extensively in \cite{rockafellar2018}. 
Examples of coherent divergences include the Chi-Square, Kullback-Leibler divergences and the squared Hellinger distance. 

Before stating the announced asymptotic guarantees, we make further assumptions on the structure of both the context and parameter spaces.

\begin{ass}[Structure] $\Theta$ and $\mathcal{X}$ respectively check:
    \begin{itemize} 
        \item $\Theta$ is a compact subset of some $\mathbb R^d$.
        \item $\mathcal{X}$ is a compact subset of some $\mathbb{R}^D$.
    \end{itemize}
  \label{ass:structure}
\end{ass}

We now state Lemma~\ref{lemma:asymptotic_results} which provides asymptotic certificate guarantees for the robust risk, asymptotically controlling the true counterfactual risk $\bE_{\xi \sim  P}[\ell(\xi;\theta)]$ with high probability. It is easy to show that under Assumptions~\ref{ass:coherence} and \ref{ass:structure}, Lemma~\ref{lemma:asymptotic_results} can be obtained by a direct application of Proposition~1 of \cite{duchi2016}. The full proof of this result is provided in Appendix~\ref{app:proof_lemma1}.

\begin{lemma}[Asymptotic guarantee - Proposition 1 of \cite{duchi2016}] For a fixed level of confidence $\delta \in (0, 1]$, we have that $\forall \theta \in \Theta$:
\begin{equation}
    \underset{n \to \infty}{\lim} \ \ \mathbb{P} \left( \mathbb E_{\xi \sim  P} \left[ \ell(\xi;\theta) \right] \le \Rrob{^{\varphi}}{\theta, \epsilon_n} \right) \ge 1-\delta
\end{equation}
where $\epsilon_n$ is defined by the $(1-\delta)$ Chi-Squared quantile, $\epsilon_n(\delta) = \chi^2_{1, 1-\delta}/n$.
\label{lemma:asymptotic_results}
\end{lemma}
\begin{proof}
The proof is deferred to Appendix~\ref{app:proof_lemma1}. It mainly consists in showing that under Assumptions~\ref{ass:coherence} and \ref{ass:structure}, the conditions for applying the Proposition 1 of \cite{duchi2016} are fulfilled.
\end{proof}

For the CRM problem, this result is of upmost importance as it allows us to control the post-decision surprise suffered for a given policy $\pi_\theta$ and allows for a pointwise control of the true risk.  

\paragraph{Remark 2} A stronger result than Lemma \ref{lemma:asymptotic_results} would control with high probability the true risk of the robustified policy $\mathbb E_{\xi \sim  P}[ \ell(\xi;\Trob{})]$ with the optimal value $\Rrob{\mU}{\Trob{}, \epsilon_n}$. 
It is easy to see that, if $P \in \mathcal{U}_\epsilon$, then $\mathbb E_{\xi\sim P} [\ell(\xi;\Trob{})] \le \Rrob{\mU}{\Trob{}, \epsilon_n}$, hence $\mathbb{P}(\mathbb E_{\xi\sim P} [\ell(\xi;\hat{\theta}_n)] \le \Rrob{\mU}{\Trob{}, \epsilon_n}) \ge \mathbb{P}(P \in \mathcal{U}_\epsilon)$. By exhibiting strong rates of convergence of the empirical distribution $\hat{P}_n$ towards the true distribution $P$, such a result could be reached. Under mild assumptions on $P$, this guarantee has been proved in \cite{esfahani2017} for Wasserstein based uncertainty sets. In our current case where $\mathcal{U}_\epsilon$ is defined by information divergences this result holds solely under the assumption that $P$ is finitely supported \cite{van2017data}, a plausible situation 
when the logging policy is defined on a finite number of (context, action) pairs. 




\subsection{Equivalences between DRO and SVP}
\label{subsec:svpisdro}
In this subsection, we focus on stressing the link between DRO and sample variance penalization schemes as used in the POEM algorithm. In Lemma~\ref{lemma:asymptotic_expansion}, we present an asymptotic equivalence between the robust risk (defined with coherent $\varphi$ divergences) and the SVP regularization used in POEM. This Lemma is a specific case of existing results, already detailed in \cite{duchi2016,namkoong2017variance} and \cite{gotoh2017,gotoh2018}.



\begin{lemma} [Asymptotic equivalence - Theorem 2 of \cite{duchi2016}] 
\label{lem:asymptotic_equivalence}
Under Assumptions~\ref{ass:coherence} and \ref{ass:structure}, for any $\epsilon \ge 0$, integer $n > 0$ and $\theta \in \Theta$ we have: 
\begin{equation}
  \begin{aligned}
    \Rrob{^{\varphi}}{\theta, \epsilon_n} = \hat{R}_n(\theta) + \sqrt{\epsilon_n V_n(\theta)} + \alpha_n(\theta),
  \end{aligned}
\end{equation}
with $\sup_{\theta}\sqrt{n}|\alpha_n(\theta)| \overset{P}{\longrightarrow}
0$ and $\epsilon_n = \epsilon/n$.
\label{lemma:asymptotic_expansion}
\end{lemma}
\begin{proof}
The proof is rather simple, as one only needs to show that the assumptions behind Theorem 2 of \cite{duchi2016} are satisfied. For the sake of completeness, this analysis is carried on in Appendix~\ref{app:proof_lemma2}.
\end{proof}
This expansion gives intuition on the practical effect of the DRO approach: namely, it states that the minimization of the robust risk $\Rrob{^{\varphi}}{\theta}$ based on coherent information divergences is \emph{asymptotically equivalent} to the POEM algorithm. This link between POEM and DRO goes further: the following Lemma states that sample-variance penalization is an \emph{exact} instance of the DRO problem when the uncertainty set is based on the chi-square divergence.

\begin{lemma}[Non-asymptotic equivalence] Under Assumption \ref{ass:structure} and for $\chi^2$-based uncertainty sets, for any $\epsilon \ge 0$ small enough, integer $ n > 0$ and $\theta \in \Theta$ we have: 
\begin{equation}
  \begin{aligned}
    \Rrob{^{\chi^2}}{\theta, \epsilon}  = \hat{R}_n(\theta) + \sqrt{\epsilon V_n(\theta)}.
  \end{aligned}
  \label{eq:svp_is_dro}
\end{equation}
\label{lemma:svp_is_dro}
\end{lemma}

The proof is detailed in Appendix \ref{app:proof_lemma3}. In the case of $\chi^2$-based uncertainty sets, the expansion of Lemma~\ref{lemma:asymptotic_expansion} holds non-asymptotically, and the POEM algorithm can be interpreted as the minimization of the distributionally robust risk $\Rrob{^{\chi^2}}{\theta, \epsilon}$. To the best of our knowledge, it is the first time that a connection is drawn between counterfactual risk minimization with SVP \cite{swaminathan2015batch} and DRO.

\subsection{Kullback-Leibler based CRM}
\label{subsec:kl_crm}
Among information divergences, we are interested in the ones that allow tractable optimization. Going towards this direction, we investigate in this subsection the robust counterfactual risk generated by Kullback-Leibler (KL) uncertainty sets and stress its efficiency. The KL divergence is a \emph{coherent} $\varphi$-divergence, with $\varphi(z) \triangleq z\log(z)+z-1$ for $z>0$ and $\varphi(z)=\infty$ elsewhere. The robust risk $\Rrob{^{\text{KL}}}{\theta, \epsilon}$ therefore benefits from the guarantees of Lemma~\ref{lemma:asymptotic_results}. Furthermore, it enjoys a simple analytic formula stated in Lemma~\ref{lemma:kl_dro} that allows for direct optimization. For conciseness reasons, the proof of this Lemma is deferred to Appendix~\ref{app:kl_dro_proof}.

\begin{lemma}[Kullback-Leibler Robustified Counterfactual Risk]
\label{lemma:kl_dro} Under Assumption \ref{ass:structure}, the robust risk defined with a Kullback-Leibler uncertainty set can be rewritten as:
\begin{align}
    \Rrob{^{\text{\upshape KL}}}{\theta, \epsilon} &= \inf_{\gamma > 0} \left( \gamma\epsilon + \gamma\log\mathbb E_{\xi \sim \Pn}[\exp(\ell(\xi;\theta)/\gamma)] \right) 
    \\  &=\mathbb{E}_{\xi\sim \Pn^{\gamma^\star}(\theta)}[\ell(\xi;\theta)].
  \label{eq:kl-dro}
\end{align}
where $\Pn^{\gamma}$ denotes the Boltzmann distribution at temperature $\gamma > 0$, defined by
$$
\Pn^\gamma(\xi_i|\theta) = \frac{\exp(\ell(\xi_i;
  \theta)/\gamma)}{\sum_{j=1}^n\exp(\ell(\xi_j; \theta)/\gamma)}.
$$
\end{lemma}

A direct consequence of Lemma~\ref{lemma:kl_dro} is that the worst-case distribution in the uncertainty set (defined by the KL divergence) takes the form of a Boltzmann distribution. Henceforth, minimizing the associated robust risk is equivalent with the optimization of the following objective:

\begin{align}
   \Rrob{^{\text{\upshape KL}}}{\theta} = \frac{\sum_{i=1}^n \ell(\xi_i;
  \theta)\exp(\ell(\xi_i; \theta)/\gamma^\star)}{\sum_{j=1}^n \exp(\ell(\xi_j; \theta)/\gamma^\star)}.
\label{eq:kl_objective}
\end{align}

From an optimization standpoint this amounts to replacing the empirical distribution of the logged data with a Boltzmann adversary which re-weights samples in order to put more mass on hard examples (examples with high cost).

In what follows, we call \firstalgo the algorithm minimizing the objective \eqref{eq:kl_objective} while treating $\gamma^\star$ as a hyper-parameter that controls the hardness of the re-weighting. A small value for $\gamma^\star$ will lead to a conservative behavior that will put more weight on actions with high propensity cost. In the limit when $\gamma^\star \to 0$, the robust risk only penalizes the action with highest propensity re-weighted cost. On the other end, a very large $\gamma^\star$ brings us back in the limit to the original CIPS estimator where the samples have equal weights. In a naive approach, this parameter can be determined through \emph{cross-validation} and kept constant during the whole optimization procedure.

Lemma~\ref{lemma:gamma} goes further into the treatment of the optimal temperature parameter $\gamma^*$ and provides an \emph{adaptive} rule for updating it during the robust risk minimization procedure.

\begin{lemma}
    The value of the optimal temperature parameter $\gamma^*$ can be approximated as follows:
    \begin{align}
        \gamma^* \approx \sqrt{\frac{V_n(\theta)}{2\epsilon}}.
    \end{align}
    \label{lemma:gamma}
\end{lemma}
The proof relies on a second-order Taylor approximation of the log moment-generating function of the loss, more precisely, of the convex one-dimensional function and is deferred to Appendix~\ref{app:gamma_proof}. This results implies that $\gamma^*$ should be updated concurrently to the parameter $\theta$ during the minimization of the robustified risk \eqref{eq:kl_objective}. This leads to an algorithm we call adaptive \firstalgo, or aKL-CRM. Pseudo-code for this algorithm is provided in Algorithm~\ref{algo:adaptive_kl_crm}. As for POEM and \firstalgo, its hyper-parameter $\epsilon$ can be determined through cross-validation. 

\begin{algorithm}
\caption{\secondalgo}
\label{algo:adaptive_kl_crm}
\SetKwBlock{Sampling}{\mdseries \textit{(Sampling)}}{}
\SetKwBlock{CMA}{\mdseries\textit{(CMA-ES iteration)}}{}
\SetKwBlock{GNN}{\mdseries\textit{(NICE iteration)}}{}
\SetKwBlock{PEN}{\mdseries\textit{(Penalty update)}}{}
\SetKwComment{Comment}{//}{}
\SetKwInOut{Input}{inputs}
\SetKwInput{HP}{hyper-parameters}
\SetKwInOut{Output}{output}
\Input{
  $\mathcal H=\{(x_1,y_1,p_1,c_1),\ldots,(x_n,y_n,p_n,c_n)\}$, parametrized
  family of policies $\pi_\theta$}
\HP{clipping constant $M$, uncertainty set size $\epsilon$}
\Repeat{convergence}
{
\textbf{compute} the counterfactual costs $z_i\leftarrow c_i\min(M,\pi_\theta(y_i\vert x_i)/p_i)$ for $i=1,\hdots,n$\\
\textbf{compute} the optimal temperature $\gamma^\star \leftarrow \sqrt{\sum_{j=1}^n\left(z_i-\bar{z}\right)^2/(2\epsilon)}$, where $\bar{z} = \sum_{j=1}^n z_i/n$\\
\textbf{compute} the normalized costs $s_i \leftarrow e_i/\sum_{j=1}^n e_j$ 
for $i=1,\hdots,n$, where $e_i = e^{z_i/\gamma^\star}$\\
\textbf{compute} the re-weighted loss $L \leftarrow \sum_{i=1}^n z_is_i$\\
\textbf{update} $\theta$ by applying an L-BFGS step to the loss $L$
}
\end{algorithm}

\section{Experimental results}
\label{sec:exp}

It is well known that experiments in the field of counterfactual reasoning are highly sensitive to differences in datasets and implementations. Consequently, to evaluate and compare the two algorithms we previously introduced to existing solutions, we rigorously follow the experimental procedure introduced in \cite{swaminathan2015batch} and used in several other works - such as \cite{swaminathan2015self} since then. It relies on a supervised to unsupervised dataset conversion \cite{agarwal2014taming} to build bandit feedback from multi-label classification datasets. As in \cite{swaminathan2015batch}, we train exponential models
$$
\pi_\theta(y\vert x) \propto \exp\left(\theta^T\phi(x,y)\right)
$$ for the CRM problem and use the same datasets taken from the LibSVM repository. For reproducibility purposes, we used the code provided by its authors \footnote{\url{http://www.cs.cornell.edu/~adith/POEM/index.html}} for all our experiments.

\subsection{Methodology} 
For any multi-label classification tasks, let us note $x$ the input features and $y^\star\in\{0,1\}^q$ the labels. The full supervised dataset is denoted $\mathcal{D}^\star\triangleq \left\{(x_1,y_1^\star),\hdots, (x_N, y_N^\star)\right\}$, and is split into three parts: $\mathcal{D}^\star_{\text{train}}$, $\mathcal{D}^\star_{\text{valid}}$, $\mathcal{D}^\star_{\text{test}}$. For every of the four dataset we consider (Scene, Yeast, RCV1-Topics and TMC2009), the split of the training dataset is done as follows: 75\% goes to $\mathcal{D}^\star_{\text{train}}$ and 25\% to $\mathcal{D}^\star_{\text{valid}}$. The test dataset $\mathcal{D}_{\text{test}}$ is provided by the original dataset. As in \cite{swaminathan2015batch}, we use joint features maps $\phi(x,y) = x\otimes y$ and train a Conditional Random Field \cite{lafferty2001conditional} (CRF) on a fraction (5\%, randomly constituted) of $\mathcal{D}^\star_{\text{train}}$. This CRF has access to the full supervised feedback and plays the role of the logging policy $\pi_0$. That is, for every $x_i\in \mathcal{D}^\star$, a label prediction $y_i$ is sampled from the CRF with probability $p_i$. The quality of this prediction is measured through the Hamming loss: $c_i = \sum_{l=1}^q \vert y_l - y_l^\star\vert$. The logged bandit dataset is consequently generated by running this policy through $\mathcal{D}^\star_{\text{train}}$ for $\Delta=4$ times ($\Delta$ is the \emph{replay count}). After training, the performances of the different policies $\pi$ are reported as their expected Hamming loss on the held-out set $\mathcal{D}^\star_{\text{test}}$. Every experiment is run 20 times with a different random seed (which controls the random training fraction for the logging policy and the creation of the bandit dataset). 

For each dataset we compare our algorithm with the naive CIPS estimator and the POEM. For all four algorithms (CIPS, POEM, \firstalgo, \secondalgo), the numerical optimization routine is deferred to the L-BFGS algorithm. As in \cite{swaminathan2015batch}, the clipping constant $M$ is always set
to the ratio of the 90\%ile to the 10\%ile of the propensity scores observed in logs $\mathcal{H}$. Other hyper-parameters are selected by cross-validation on $\mathcal{D}^\star_{\text{valid}}$ with the unbiased counterfactual estimator \eqref{eq:ips_empirical_risk}. In the experimental results, we also report the performance of the logging policy $\pi_0$ on the test set as an indicative baseline measure, and the performance of a skyline CRF trained on the whole supervised dataset, despite of its unfair advantage.

\subsection{Results}

\begin{table}[t]
    \centering
    \begin{tabular}{c?c|c|c|c?}
         & \textbf{Scene} & \textbf{Yeast} & \textbf{RCV1-Topics} & \textbf{TMC2009}   \\
         \Xhline{1.5pt}
         $\pi_0$ & 1.529 & 5.542 & 1.462 & 3.435\\
         \hline
         \hline
         CIPS & 1.163 & 4.658 & 0.930 & 2.776\\
         POEM & 1.157 & \textbf{4.535}& 0.918 & 2.191\\
         \firstalgo & 1.146 & 4.604 & 0.922 & 2.136\\
         \secondalgo & \textbf{1.128} & \textbf{4.553} & \textbf{0.783}& \textbf{2.126} \\
         \hline
         \hline
         CRF & 0.646 & 2.817& 0.341 & 1.187
    \end{tabular}
    \caption{Expected Hamming loss on $\mathcal{D}^*_{\text{test}}$ for the different algorithms, averaged over 20 independent runs. Bold font indicate that one or several algorithms are statistically better than the rest, according to a one-tailed paired difference t-test at significance level of 0.05.}
    \label{tab:results}
\end{table}

Table~\ref{tab:results} reports the expected Hamming loss of the policies obtain with different algorithms on the Scene, Yeast, RCV1-Topics and TMC2007 dataset, averaged on 20 random seeds. The results reported for the baselines are coherent with \cite{swaminathan2015batch}. On each dataset, \secondalgo comes out at one of the best algorithm (according to a one-tailed paired difference t-test at significance level 0.05) and outperforming the POEM baseline on three out of four datasets. The results for \firstalgo are more mitigated: it outperforms POEM on two datasets, but shows weaker performance on the two others clearly stressing the efficiency of an adaptive temperature parameter. 

As in \cite{swaminathan2015batch}, we can further evaluate the quality of the learned policies by evaluating the Hamming loss of their \emph{greedy} version (selecting only the arm that is attributed the most probability mass by the policy). This comes at less expense that sampling from the policy as this does not require to compute the normalizing constant in \eqref{eq:exponential_model}. These results are reported in Table~\ref{tab:greedy_results}, and are consistent with the conclusions of Table~\ref{tab:results}. One can note that the improvement brought by \secondalgo over POEM is even sharper under this evaluation. 

Another experiment carried in \cite{swaminathan2015batch} focuses on the size of the bandit dataset $\mathcal{H}$. This quantity can be easily modulated by varying the \emph{replay count} $\Delta$ - the number of times we cycle through $\mathcal{D}^\star_{\text{train}}$ to create the logged feedback $\mH$. Figure~\ref{fig:rep_count} reports the expected Hamming loss of policies trained with the POEM, \firstalgo and \secondalgo algorithms for different value of $\Delta$, ranging from 1 to 256, and based on the Yeast and Scene datasets. Results are averaged over 10 independent runs. For large values of $\Delta$ (that is large bandit dataset) all algorithms seem to confound; this is to be expected as Lemma \ref{lem:asymptotic_equivalence} states that for any coherent $\varphi$-divergences, all robust risks are asymptotically equivalent. It also stands out that for small replay counts (i.e the small data regime, a more realistic case), the KL-based algorithms outperform POEM.

\begin{table}
    \centering
    \begin{tabular}{c?c|c|c|c?}
         & \textbf{Scene} & \textbf{Yeast} & \textbf{RCV1-Topics} & \textbf{TMC2009}   \\
         \Xhline{1.5pt}

         CIPS & 1.163 & 4.369 & 0.929 & 2.774\\
         POEM & 1.157 & \textbf{4.261}& 0.918 & 2.190\\
         \firstalgo & 1.146 & 4.316 & 0.922 & 2.134\\
         \secondalgo & \textbf{1.128} & \textbf{4.271} & \textbf{0.779}& \textbf{2.034}
    \end{tabular}
    \caption{Hamming loss on $\mathcal{D}^*_{\text{test}}$ for the different greedy policies, averaged over 20 independent runs. Bold font indicates that one or several algorithms are statistically better than the rest, according to a one-tailed paired difference t-test at significance level of 0.05.}
    \label{tab:greedy_results}
\end{table}

\begin{figure}[t]
    \centering
    \begin{subfigure}{0.45\linewidth}
        \includegraphics[width=\linewidth]{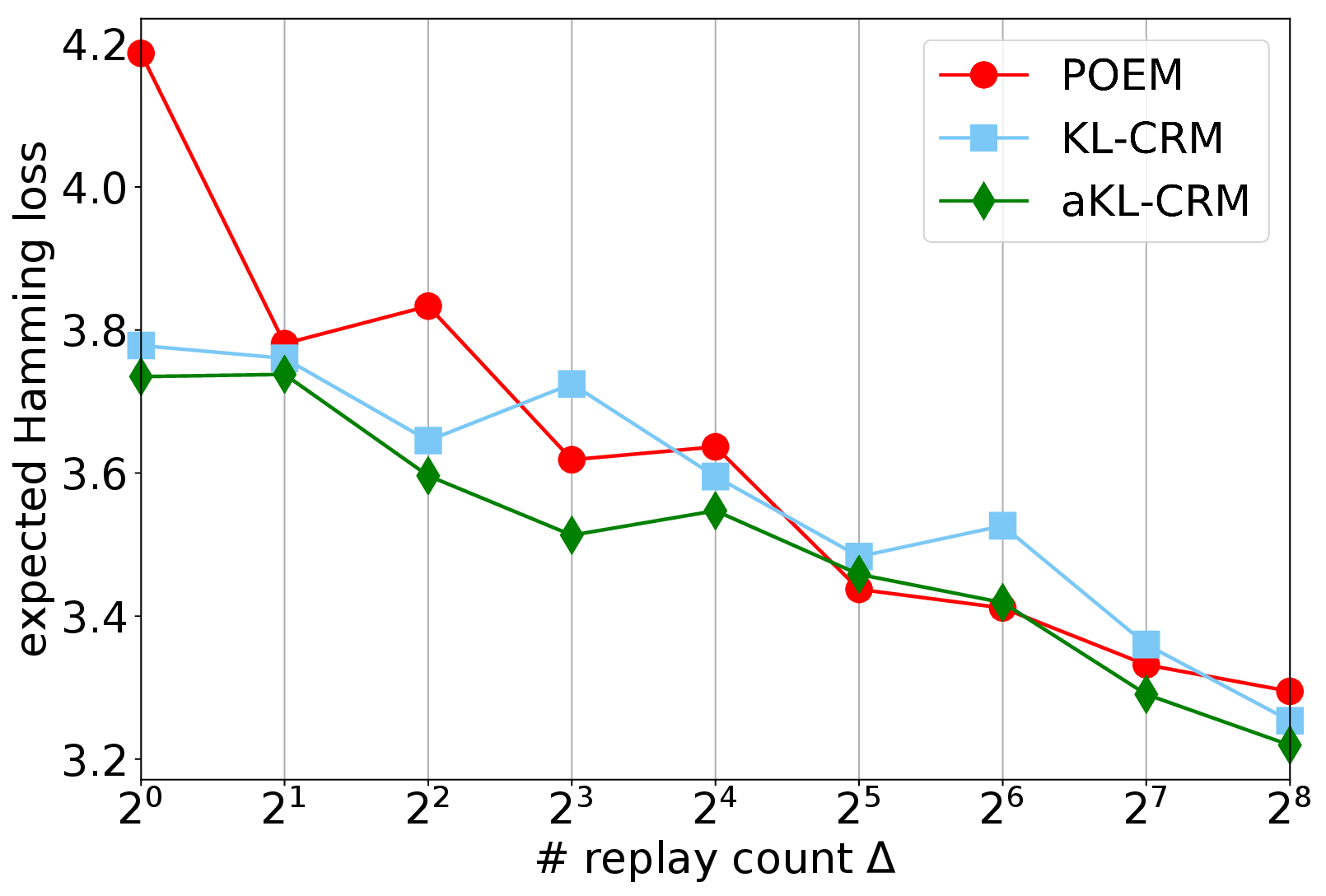}
        \caption{Yeast dataset}
    \end{subfigure}
    \begin{subfigure}{0.45\linewidth}
        \includegraphics[width=\linewidth]{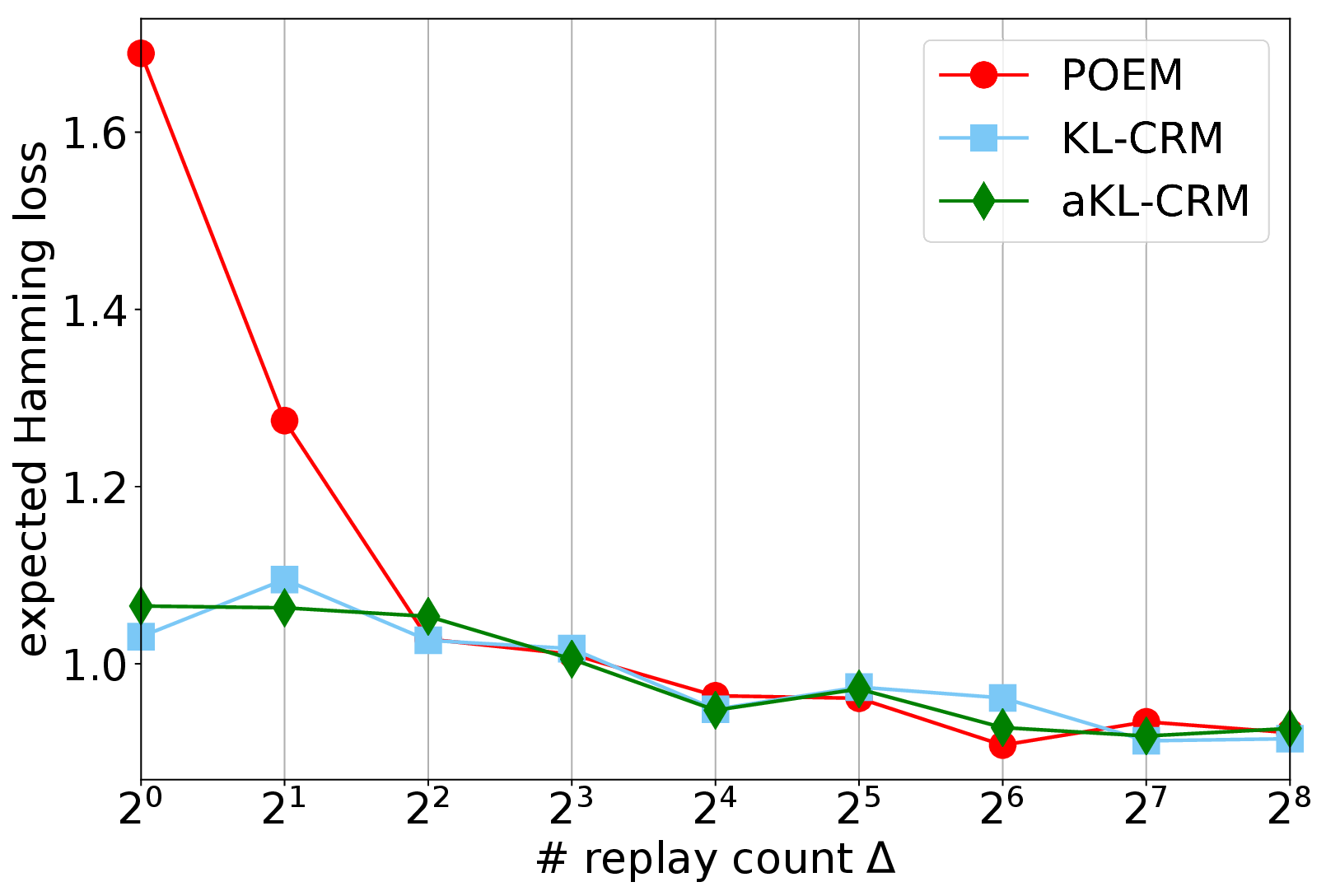}
        \caption{Scene dataset}
    \end{subfigure}
    \caption{Impact of the replay count $\Delta$ on the expected Hamming loss. Results are average over 10 independent runs, that is 10 independent train/test split and bandit dataset creation. KL-CRM and aKL-CRM outperform POEM in the small data regime.}
    \label{fig:rep_count}
\end{figure}

\section{Conclusion}
\label{sec:conclusion}
We presented in this work a unified framework for counterfactual risk minimization based on the distributionally robust optimization of policies and motivated it by asymptotic guarantees available when the uncertainty measure is based on $\varphi$-divergences. We showed that this new framework generalizes existing solutions, like sample-variance penalized counterfactual risk minimization algorithms \cite{swaminathan2015batch}. Our work therefore opens a new avenue for reasoning about counterfactual optimization with logged bandit feedback as we showed that a KL-divergence based formulation of the counterfactual DRO problem can lead to tractable and efficient algorithms for the CRM problem, outperforming state-of-the-art results on a collection of datasets. 

The authors of \cite{swaminathan2015batch} also proposed a modification to the POEM algorithm that can be optimized at scale using stochastic gradient descent. Future work should therefore aim at developing stochastic optimization schemes for both \firstalgo and \secondalgo so that they could handle large datasets. From the perspective of experimental evaluation, measuring the impact of the DRO formulation on the doubly robust \cite{dudik2011doubly} and the self-normalized \cite{swaminathan2015self} estimators would further validate its relevance for real world problems. Finally, a more theoretical line of work could focus on proving finite-samples performance certificate guarantees for distributionally robust estimators based on coherent $\varphi$-divergences, further motivating their use for counterfactual risk minimization. 

\bibliography{biblio.bib}
\bibliographystyle{plain}

\newpage
\appendix

\section{Proof of Lemma \ref{lemma:asymptotic_results}} \label{app:proof_lemma1}

To prove this lemma, we will need the following result:
\begin{lemma}
    \label{lemma:assumptions_duchi}
    Under Assumption \ref{ass:structure}, we have that:
    \begin{itemize}
        \item There exists a measurable function $M: \mX\times\mY \to \mathbb{R}^+$ such that for all $\xi \in \mX\times \mY, \ell(\xi, \cdot)$ is $M(\xi)$-Lipschitz with respect to some norm $||.||$ in $\Theta$.
        \item $\mathbb{E}_{\xi \sim P} \left[M(\xi)^2\right] < \infty$ and $\mathbb{E}_{\xi \sim P}\left[\ell(\xi, \theta)^2 \right] < \infty$ for $\theta \in \Theta$.
    \end{itemize}
     where the distribution $P=\nu\otimes \pi$ and the loss $\ell(\xi,\theta)\triangleq c(x,y)\min\left(M, \frac{\pi_\theta(y \vert x)}{\pi_0(y \vert x)}\right)$ and $\xi=(x,y)$.
\end{lemma}
\begin{proof}
We recall that the policy space is restricted to exponential policies such that $\pi_\theta(y\vert x) \propto \exp\left( \theta^T \phi(x, y)\right)$ and that under Assumption \ref{ass:structure} the parameter space $\Theta$ is compact. Let $f(\xi, \theta) \triangleq c(x,y) \frac{\pi_\theta(y \vert x)}{\pi_0(y \vert x)}$. We have that:
\begin{align*}
     \left\vert\ell(\xi,\theta_1)-\ell(\xi,\theta_2)\right\vert \leq \left\vert f(\xi,\theta_1)-f(\xi,\theta_2)\right\vert, \quad \forall \theta_1, \theta_2\in\Theta
\end{align*}
Note that $f(\xi,\cdot)$ is continuous derivable on $\Theta$ compact. We note:
\begin{align*}
    M(\xi) \triangleq \underset{\theta \in \Theta}{\max} \ \lVert\nabla_\theta f(\xi, \theta)\rVert \end{align*}
$f(\xi,\cdot)$ is therefore $M(\xi)$-Lipschitz in $\theta$ w.r.t. $||.||$ and, consequently so is $\ell(\xi,\cdot)$. As the function $M$ is continuous, by boundedness of the context space $\mathcal{X}$ (Assumption \ref{ass:structure}), we have $\mathbb{E}_{\xi \sim P} [M(\xi)^2] < \infty$ and $\mathbb{E}_{\xi \sim P} [\ell(\xi, \theta)^2] < \infty$ for $\theta \in \Theta$.
\end{proof}

Under Assumptions \ref{ass:coherence} and \ref{ass:structure} and using the result in Lemma \ref{lemma:assumptions_duchi}, we now verify all the assumptions of Proposition 1 in \cite{duchi2016}. Lemma \ref{lemma:asymptotic_results} is consequently a direct application of the latter.

\paragraph{Remark} Note that this result does not apply only to exponentially parametrized policies. Indeed, sufficient conditions for Lemma \ref{lemma:assumptions_duchi} would be milder - we only used that $\pi_\theta$ is continuously derivable w.r.t. $\theta$.

\section{Proof of Lemma \ref{lemma:asymptotic_expansion}}\label{app:proof_lemma2}

Under Assumptions \ref{ass:coherence} and \ref{ass:structure} and using Lemma \ref{lemma:assumptions_duchi}, Lemma \ref{lemma:asymptotic_expansion} is a direct application of Theorem 2 of in \cite{duchi2016}.

\section{Proof of Lemma \ref{lemma:svp_is_dro}}\label{app:proof_lemma3}

The line of proof is similar to the proof of Theorem 3.2 of \cite{gotoh2018}. To ease notations, we denote $Z\triangleq \ell(\xi,\theta)$. Let's consider $\varphi$ a coherent information divergence and $\varphi^*(z) \triangleq \underset{t>0}{\sup} \ zt-\varphi(t)$ its \textit{convex conjugate}. By strong duality:
\begin{eqnarray*}
    \sup_{D_\varphi(Q\|\hat{P}_n) \le \epsilon}\mathbb E_{Q}[Z] =
    \inf_{\gamma \ge 0}\gamma\epsilon + \sup_Q \ (\mathbb
    E_{Q}[Z]-\gamma D_\varphi(Q\|\hat{P}_n))
\end{eqnarray*}
The Envelope Theorem of \cite{rockafellar2018} states that:
$$
\begin{aligned}
\sup_Q \ (\mathbb
    E_{Q}[Z]-\gamma D_\varphi(Q\|\hat{P}_n))& = \inf_{c \in \mathbb R} (c + \gamma \mathbb E_{\hat{P}_n}[\varphi^*((Z-c)/\gamma)])
\end{aligned}
$$
Hence we obtain that:
$$\begin{aligned}
    \sup_{D_\varphi(Q\|\hat{P}_n) \le \epsilon}\mathbb E_{Q}[Z] = &\inf_{\gamma \ge 0}\gamma\epsilon + \inf_{c \in \mathbb R}(c + \gamma \mathbb E_{\hat{P}_n}[\varphi^*((Z-c)/\gamma)])
\label{equation_proposition31}
\end{aligned}$$

In the specific case of the modified $\chi^2$-divergence, let $\varphi(z)= (z-1)^2$ and its convex conjugate is $\varphi^\star(z) = z + z^2/4$ for $z>-2$. Assume for now that 
$\Big\{(Z-c)/\gamma>-2\Big\}$ holds almost surely. 
Solving 
$$\inf_{c \in \mathbb R}c + \gamma \mathbb E_{\hat{P}_n}\left[\frac{Z-c}{\gamma}+\frac{(Z-c)^2}{4\gamma^2}\right]$$ leads to:
$$
\begin{aligned}
    \sup_{D_\varphi(Q\| \hat{P}_n) \le \epsilon}\mathbb E_{Q}[Z] &= \mathbb E_{\hat{P}_n}[Z] + \inf_{\gamma \ge
    0}\left(\gamma\epsilon + \frac{1}{4\gamma} V_n(Z) \right)\\
  &= \mathbb E_{\hat{P}_n}[Z] + \sqrt{\epsilon V_n(Z)}
\end{aligned}
$$
This yields the announced result in Lemma~\ref{lemma:svp_is_dro}. One can check that the minimum (w.r.p to $\gamma$) of $\gamma\epsilon + \frac{1}{4\gamma} V_n(Z)$ is attained at $\gamma^*=\frac{1}{2}\sqrt{V_n(Z)/\epsilon}$. Therefore for $\epsilon$ small enough, $\gamma^*$ is big enough to ensure that indeed the event $\Big\{(Z-c)/\gamma^*>-2\Big\}$ holds almost surely (with respect to the measure $P_n$).

\section{Proof of Lemma~\ref{lemma:kl_dro}}

\label{app:kl_dro_proof}

To ease notations, we denote $Z\triangleq \ell(\xi,\theta)$. As in the proof of Lemma~\ref{lemma:svp_is_dro}, one can obtain that:
$$
\begin{aligned}
    \max_{KL\left(Q \vert\vert \hat{P}_n\right) \le
  \epsilon}\mathbb E_Q&[\ell(\xi;\theta)] = \inf_{\gamma \ge 0}\Big\{\gamma\epsilon  + \inf_{c \in \mathbb R}\left(c +
    \gamma\mathbb E_{\hat{P}_n}[\varphi^*_{\text{KL}}((Z-c)/\gamma)]\right)\Big\}
\end{aligned}
$$
Since $\varphi_{\text{KL}}(z) = z\log z -z +1$, it is a classical convex analysis exercise to show that its convex conjugate is $\varphi^*_\text{KL}(z)=e^z -1$. Therefore:
$$
\begin{aligned}
\max_{KL\left(Q \vert\vert \hat{P}_n\right) \le
  \epsilon}\mathbb E_Q&[\ell(\xi;\theta)] = \inf_{\gamma \ge 0}\Big\{\gamma\epsilon  +
  \inf_{c \in \mathbb R}\left(c +
    \gamma\mathbb E_{\hat{P}_n}[e^{(Z-c)/\gamma}-1]\right)\Big\}
\end{aligned}$$
Solving 
$$
\inf_{c \in \mathbb R}\left(c +
    \gamma\mathbb E_{\hat{P}_n}[e^{(Z-c)/\gamma}-1]\right)
$$
is straightforward and leads to:
\begin{align}
    \max_{KL\left(Q \vert\vert \hat{P}_n\right) \le
  \epsilon}\mathbb E_Q[\ell(\xi;\theta)] = \inf_{\gamma \ge 0}\left\{\gamma\epsilon + \gamma\log\mathbb
  E_{\hat{P}_n}\left[e^{Z/\gamma}\right]\right\}
  \label{eq:solving_for_gamma}
\end{align}
Differentiating the r.h.s and setting to 0 shows that the optimal $\gamma$ is solution to the fixed-point equation:
\begin{align}
\gamma = \frac{\mathbb E_{\hat{P}_n^\gamma}[Z]}{\epsilon +
  \log(\mathbb E_{\hat{P}_n}[e^{Z/\gamma}])} 
\end{align}
where $d\hat{P}_n^\gamma(z) := (1/\mathbb E_{\hat{P}_n}[e^{z/\gamma}])e^{z/\gamma}d\hat{P}_n(z)$
is the density of the Gibbs distribution at temperature $\gamma$ and state
degeneracies $\hat{P}_n$. Replacing this value for $\gamma$ in the r.h.s of \eqref{eq:solving_for_gamma} yields the announced result. This formula has also been obtained in \cite{hu2013kullback}, using another line of proof.

\section{Proof of Lemma~\ref{lemma:gamma}}
\label{app:gamma_proof}
Using notations from Appendix~\ref{app:kl_dro_proof}, consider the log moment generating function :
\begin{align}
    \Phi:\alpha \to \log\bE_{\Pn}\left[e^{Z\alpha}\right]
\end{align}
well defined as $\Pn$ has finite support and $Z$ is bounded a.s. It checks the following equalities:
$$
\begin{aligned}
    \Phi(0) &= 0 \\
    \Phi'(0) &= \bE_{\Pn}\left[Z\right]\\
    \Phi''(0) &= V_n(Z)
\end{aligned}
$$
and a second-order Taylor expansion around 0 yields:
$$
\Phi(\alpha) = \alpha \bE_{\Pn}\left[Z\right] + \frac{\alpha^2}{2}V_n(Z) + o_{0}(\alpha^2)
$$
With $\alpha=1/\gamma$ and injecting this result in the r.h.s of Equation~\eqref{eq:solving_for_gamma} yields:
\begin{align}
    \max_{KL\left(Q \vert\vert \hat{P}_n\right) \le
  \epsilon}\mathbb E_Q[\ell^M(\xi;\theta)] = \inf_{\gamma \ge 0}\left\{\gamma\epsilon + \bE_{\Pn}\left[Z\right]  +\frac{V_n(Z)}{2\gamma}+o_{\infty}(1/\gamma)\right\}
\end{align}
Solving (approximately) the r.h.s of the above equation yields the following relation:
$$
    \gamma \simeq \sqrt{\frac{V_n(Z)}{2\epsilon}}
$$
as announced by Lemma~\ref{lemma:gamma}.

\section{Experimental details}
\label{app:exps}
We hereinafter provide details on the experimental procedure we followed to obtain the results announced in Section~\ref{sec:exp}. This procedure is the one described in \cite{swaminathan2015batch} - we used the code provided by its authors \footnote{\url{http://www.cs.cornell.edu/~adith/POEM/index.html}} for our experiments.

For every of the four dataset we consider (Scene, Yeast, RCV1-Topics and TMC2009), the dataset is randomly split in two parts: 75\% goes to $\mathcal{D}^\star_{\text{train}}$ and 25\% to $\mathcal{D}^\star_{\text{valid}}$. The test dataset $\mathcal{D}_{\text{test}}$ is provided by the original dataset. The logging policy is a CRF is then trained on a fraction ($f$=5\%) of the training dataset $\mathcal{D}^\star_{\text{train}}$. The logged bandit dataset is generated by running this policy through $\mathcal{D}^\star_{\text{train}}$ for $\Delta=4$ times ($\Delta$ is the \emph{replay count}).

As in \cite{swaminathan2015batch,swaminathan2015batch,swaminathan2015self} we train linearly parametrized exponential policies:
$$
    \pi_\theta(y\vert x) = \frac{\exp\left(\theta^T\phi(x,y)\right)}{\mathbb{Z}(x)}
$$
where ${\mathbb{Z}(x)}$ is a normalization constant. We use the same joint feature map $\phi(x,y)=x\otimes y$ with $y$ being a bit vector $\left\{0,1\right\}^L$ where $L$ describes the number of labels of the considered multi-label classification task.

We optimize the objective of each algorithm (CIPS, POEM, \firstalgo, \secondalgo) with the L-BFGS routine. Their respective hyper-parameters are determined via cross-validation on $\mathcal{D}^\star_{\text{valid}}$ with the unbiased estimator \eqref{eq:ips_empirical_risk}. The different ranges used for this cross-validation step are detailed in Table~\ref{tab:hp_range}. 

\begin{table}[h]
    \centering
    \begin{tabular}{c|c|c}
         \textbf{Algorithm} & \textbf{Hyper-parameter} & \textbf{Logarithmic range}  \\
         \hline 
          CIPS & - & - \\
          POEM & $\lambda$ & [-6,0] \\
          \firstalgo & $\gamma$ & [-3,4] \\
          \secondalgo & $\epsilon$ & [-6, 0]
    \end{tabular}
    \caption{Logarithmic ranges for the cross-validation of hyper-parameters.}
    \label{tab:hp_range}
\end{table}

\end{document}